\documentclass[letterpaper, 10 pt, conference]{ieeeconf}  

\IEEEoverridecommandlockouts                              
\usepackage{amsmath,amsfonts,bm}

\def\1{\bm{1}}

\DeclareMathAlphabet{\mathsfit}{\encodingdefault}{\sfdefault}{m}{sl}
\SetMathAlphabet{\mathsfit}{bold}{\encodingdefault}{\sfdefault}{bx}{n}

\def\cala{{\cal A}}
\def\calf{{\cal F}}
\def\calk{{\cal K}}
\def\calx{{\cal X}}
\def\rea{\mathbb{R}}

\def\intnum{\mathbb{N}}

\def\begmat#1{\begin{bmatrix}#1\end{bmatrix}}

\usepackage[prependcaption,colorinlistoftodos]{todonotes}

\usepackage{hyperref}
\usepackage{url}
\usepackage{soul,xcolor}
\usepackage{pifont}
\newcommand{\cmark}{\text{\ding{51}}}
\newcommand{\xmark}{\text{\ding{55}}}

\usepackage{array}
\newcolumntype{M}[1]{>{\centering\arraybackslash}m{#1}}

\usepackage{mathtools}
\DeclarePairedDelimiter{\norm}{\lVert}{\rVert}
\DeclarePairedDelimiter{\vecnorm}{\lvert}{\rvert}

\newtheorem{theorem}{Theorem}
\newtheorem{definition}{Definition}

\newtheorem{remark}{Remark}

\newtheorem{proposition}{Proposition}

\usepackage[ruled, linesnumbered]{algorithm2e}
\usepackage{comment}

\usepackage{subcaption}
\graphicspath{{figures/}}

\title{\LARGE \bf
Learning Stable Koopman Embeddings}

\author{Fletcher Fan, Bowen Yi, David Rye, Guodong Shi and Ian R. Manchester  
\thanks{This work was supported by the Australian Research Council.}
\thanks{The authors are with the Australian Center for Field Robotics and Sydney Institute for Robotics and Intelligent Systems, 
        The University of Sydney, NSW 2006, Australia. Corresponding author's email:
        {\tt\small f.fan@acfr.usyd.edu.au}}%
}

\begin{document}
\thispagestyle{empty}
\pagestyle{empty}

\maketitle

\begin{abstract}
	In this paper, we present a new data-driven method for learning stable models of nonlinear systems. Our model lifts the original state space to a higher-dimensional linear manifold using Koopman embeddings. Interestingly, we prove that every discrete-time nonlinear contracting model can be learnt in our framework. Another significant merit of the proposed approach is that it allows for \emph{unconstrained} optimization over the Koopman embedding and operator jointly while enforcing stability of the model, via a direct parameterization of stable linear systems, greatly simplifying the computations involved. We validate our method on a simulated system and analyze the advantages of our parameterization compared to alternatives.
\end{abstract}

\section{Introduction}
The problem of fitting models to data generated from dynamical systems, known as system identification, is ubiquitous in science and engineering. One important consideration in system identification is the stability of the model. In many applications, the fitted model is used for prediction of future behaviors of the system, and an unstable model would erroneously produce unbounded predictions. 

There are many different forms of stability for nonlinear systems. In this paper, we consider contraction, also known as incremental stability, which can be viewed as a ``strong'' type of stability for dynamical systems that studies the convergence between \emph{any} two trajectories of the given system \cite{LohmillerSlotine1998Contraction}. There has been much prior work on learning contracting models in system identification for various model classes, including polynomial models \cite{TobenkinEtAl2017Convex, UmenbergerManchester2019Convex}, Gaussian mixture models \cite{RavichandarEtAl2017Learning} and neural network models \cite{RevayEtAl2021Convex, RevayEtAl2021Recurrent, manchester2021contractionbased}. In this paper, we propose a new class of contracting nonlinear models that combines the expressiveness of neural networks with the strong stability guarantees associated with linear systems. Furthermore, we propose a learning framework that fits this class of models to data via an \textit{unconstrained} optimization problem.

Our work bridges the gap between learning stable nonlinear models and approximating the Koopman operator \cite{Koopman1931Hamiltonian}, an infinite-dimensional \textit{linear} operator that can describe the dynamics of any nonlinear system by embedding it in a higher-dimensional space. There has been growing interest in data-driven methods that estimate finite-dimensional approximations of the Koopman operator and its eigenfunctions \cite{Schmid2010Dynamic, WilliamsEtAl2015Data, HaseliCortes2021Learning}, motivated by the appeal of being able to apply linear systems analysis to complex nonlinear systems. Due to this merit, Koopman-based methods have been developed for system identification \cite{MauroyGoncalves2020Koopman}, state observation and control \cite{KordaMezic2018Linear} of nonlinear systems. 

In Koopman identification approaches, a central problem is how to learn the Koopman embedding from data. Recently many methods  \cite{LiEtAl2017Extended, LuschEtAl2018Deep, MardtEtAl2018VAMPnets, OttoRowley2019Linearly, PanDuraisamy2020Physics, TakeishiEtAl2017Learning,  YeungEtAl2019Learning} have been proposed to address this problem, however most of them do not consider the stability of the learned model. Unstable learned models may have serious robustness issues, particularly when applied to dissipative physical systems, making them unsuitable for practical use. We aim to address this issue by imposing stability constraints on the Koopman model.

Our model class is motivated by recent work \cite{YiManchester2021equivalence} showing that, for continuous-time (CT) nonlinear systems, there is an equivalence between the Koopman and contraction approaches for stability analysis under some mild technical assumptions. We extend this equivalence result to discrete-time (DT) systems in this paper, and provide an algorithmic framework for learning Koopman models with contracting properties.

The main contributions of this paper are threefold:
\begin{itemize}
\item[\bf C1] We propose a novel Koopman learning framework that jointly models the Koopman operator and embedding from data, while imposing the model stability/contraction constraint.
\item[\bf C2] We prove that every nonlinear discrete-time contracting model can be learnt in our framework in an arbitrarily large compact set, which may be viewed as the extension of \cite{YiManchester2021equivalence} from CT to DT; see Theorem \ref{thm:1}.
\item[\bf C3] Our work builds on the contracting model class identified in \cite{RevayEtAl2021Recurrent}, which allows for \emph{unconstrained optimization} of objective functions, unlike some existing parameterizations of Koopman operators, {e.g.} \cite{MamakoukasEtAl2020Memory}. As a result, it significantly simplifies the implementation of optimization algorithms for learning the model parameters.
\end{itemize}

The rest of the paper is organized as follows. Section \ref{sec:problem_definition} defines the system identification problem and provides the background on Koopman operator theory and data-driven methods for estimating the Koopman operator, and also restates the main result in \cite{YiManchester2021equivalence}. Section \ref{sec:koopman_stability} extends \cite[Theorem 1]{YiManchester2021equivalence} to DT systems. Section \ref{sec:model_set} defines the model set used in our learning framework, and Section \ref{sec:learning} defines the optimization problem. Section \ref{sec:examples} provides some numerical validations of our framework on a handwriting dataset.

{\em Notation.} All mappings and functions are assumed sufficiently smooth. Given $f: \rea^n \to \rea^m$, we denote the gradient operator $\nabla f:=(\partial f /\partial x)^\top$. Let $\calf$ be the space of smooth real-valued scalar functions $\rea^n \to \rea$. We use $(\cdot)^\dagger$ to denote the Moore-Penrose pseudoinverse of a matrix. $\lambda_{\tt min}(\cdot)$ and $\lambda_{\tt max}(\cdot)$ respectively represent the smallest and largest eigenvalue of a square matrix. We use $\vecnorm{\cdot}$ to denote vector norms, i.e. $\vecnorm{\cdot}_2$ is the vector 2-norm. Sometimes, we may simply write $x(t)$ as $x_t$.

\section{Problem Definition and Background} \label{sec:problem_definition}

In this paper, we consider the identification of a Koopman embedding and operator for a discrete-time (DT) autonomous state-space system:
\begin{equation}\label{eqn:dtsys}
x(t+1) = f(x(t)),
\end{equation}
where $ x \in \mathbb{R}^n$ and $t$ is the timestep. We assume the system \eqref{eqn:dtsys} has a single equilibrium at $x_\star$, i.e. $f(x_\star) = x_\star$. Further, we assume the dynamics $f(x)$ are unknown, but we have access to full-state trajectory data $\lbrace \tilde{x}_t \rbrace^T_{t=0}$, generated by system \eqref{eqn:dtsys}. We are concerned with learning a function $\phi(x)$ (i.e. the Koopman embedding) that smoothly maps from the original state space $\rea^n$ to a possibly higher-dimensional space $\rea^N$ ($N\geq n$), as well as a linear matrix $\cala \in \rea^{N\times N}$ (i.e. a finite-dimensional approximation of the Koopman operator) that describes the evolution of $\phi(x)$ over time.

The Koopman embedding and the matrix $\cala$ in fact form a predictive model of the system \eqref{eqn:dtsys}, which we define as a Koopman model.
\begin{definition}[DT Koopman model] \label{def:koopman_mode}
Given a Koopman embedding $\phi(x)$ and matrix $\cala$, the corresponding Koopman model is:
\begin{equation}\label{eqn:model}
    x(t) = a(x_0, t) = \phi^L(\cala^t \phi(x_0)),
\end{equation}
where $\phi^L: \rea^N \to \rea^n$ is a left-inverse of $\phi(x)$ such that $\phi^L(\phi(x)) = x$, and $x_0$ is an initial condition.
\end{definition}

 The problem of learning $\phi(x)$ and $\cala$ can be treated as a minimization of the prediction error of the Koopman model on the given data $\lbrace \tilde{x}_t \rbrace^T_{t=0}$. 

\subsection{Koopman Operator Theory} \label{subsec:koopman}

Before presenting our theoretical contributions, we provide some background on the Koopman operator. The Koopman operator was proposed in \cite{Koopman1931Hamiltonian} for CT dynamical models. First, let us recall the definition of its variant for DT systems. 

\begin{definition}
\rm\label{def:koopman}
({\em Koopman operator}) For the DT dynamical model \eqref{eqn:dtsys}, the Koopman operator $\calk: \calf \to \calf$ is defined by
\begin{equation}
\label{koopman}
\calk[\varphi(x)] := \varphi \circ f(x)
\end{equation}
for $\varphi \in \calf$, assuming that the system has a unique solution $\forall t\in \intnum$. We term the scalar real-valued function $\varphi: \rea^n \to \rea$ an \emph{observable}.
\end{definition}

 Since the Koopman operator is defined on the functional space, it is infinite-dimensional. It is also easy to verify that the Koopman operator is linear, i.e. $\calk[k_1\varphi_1 + k_2\varphi_2] = k_1\calk[\varphi_1] + k_2 \calk[\varphi_2]$ for any $k_1,k_2 \in \rea$ and $\varphi_1,\varphi_2 \in \calf$. 
This property makes 
 Koopman methods widely popular in the analysis of dynamical models. Despite the infinite dimension of the Koopman operator, some key properties of a given nonlinear dynamical model---e.g. stability and dynamical behaviours---can be captured by a few particular functions, i.e. the Koopman eigenfunctions. 

\begin{definition}
\label{def:eigen}\rm ({\em Koopman eigenfunction})
A Koopman eigenfunction is a non-zero observable $\phi_\lambda \in \calf/\{0\}$ satisfying 
\begin{equation}
\label{eq:eigen}
\calk[\phi_\lambda(x)] = \lambda \phi_\lambda(x)
\end{equation}
for some $\lambda \in \mathbb{C}$, which is the associated Koopman eigenvalue.
\end{definition}

A Koopman eigenfunction defines a coordinate in which the system trajectories behave as a linear system. To be precise, define a coordinate $z_\lambda = \phi_\lambda(x)$, the dynamics of which 
are given by
$$
\begin{aligned}
z_\lambda(t+1) =  \lambda  z_\lambda(t), 
\end{aligned}
$$
with the initial condition $z_\lambda(0) = \phi_\lambda(x(0))$. Indeed, the definition \eqref{eq:eigen} is equivalent to solving the algebraic equation
$$
\phi_\lambda(f(x)) = \lambda \phi_\lambda(x), \quad \forall x \in \rea^n
$$
if the DT dynamical model \eqref{eqn:dtsys} is prior. 


\subsection{Dynamic Mode Decomposition} \label{subsec:dmd}

In the problem of system identification, we are more interested in finding the Koopman eigenfunctions and eigenvalues only from the collected data set $\{\tilde{x}_t\}_{t=0}^T$, for which dynamic mode decomposition (DMD) provides an efficient data-driven approach to approximating the Koopman operator \cite{Schmid2010Dynamic}.

In DMD, usually some heuristically \emph{predetermined}, sufficiently rich observables $\phi_1,\ldots \phi_N$ ($N \gg n$)---rather than Koopman eigenfunctions---are involved to learn the nonlinearity in the dynamical model. The task in the DMD method is to seek a matrix $\cala \in \rea^{N\times N}$ in order to obtain a finite-dimensional approximation of $\calk$, which minimizes the following:
\begin{equation}
\label{min:dmd}
\sum_{j=0}^T |\phi(x(t+1)) - \cala \phi(x(t)))|^2_2,
\end{equation}
in which we have defined $\psi:=[\phi_1,\ldots, \phi_N]^\top$. The least square problem \eqref{min:dmd} has a unique solution
\begin{equation} \label{eqn:dmd_solution}
    \cala = Y_1 Y_2^\dagger
\end{equation}
with 
$$
\begin{aligned}
Y_1 & := \begmat{\phi(x(1)) & \ldots & \phi(x(T))}
\\
Y_2 & := \begmat{\phi(x(0)) & \ldots & \phi(x(T-1))}
\end{aligned}
$$
if $Y_2$ is full row rank. DMD is a simple, efficient method to approximate the Koopman operator, but two issues arise: 
\begin{itemize}
\item[1)] In the DMD method, the observables $\psi$ are predetermined, which significantly affects the learning accuracy, but in the literature the selection of observables usually done in a heuristic manner. Since these observables are closely connected to the Koopman eigenfunctions for a given dynamical model, a natural question is: can the observables and the matrix $\cala$ be learnt concurrently to improve accuracy?

\item[2)] For a stable dynamical model, the above least square solution may yield an unstable model due to various kinds of perturbations in the data set $\lbrace \tilde{x}_t \rbrace^T_{t=0}$, which would be unacceptable in many applications. Hence, imposing stability constraints is an important consideration in learning algorithms.
\end{itemize}

The main motivation of the paper is to address the above issues and present a novel Koopman learning framework.

\subsection{Contraction analysis} \label{subsec:contraction}

In this paper, we are interested in stable nonlinear models. Indeed, there are many different forms of stability for nonlinear systems; we focus on contracting systems \cite{LohmillerSlotine1998Contraction}. 

Contraction analysis provides another way to study nonlinear systems by means of linear systems theory \emph{exactly} and \emph{globally}. In contraction analysis we are concerned with the differential dynamics of a given system, which is indeed a linear time-varying (LTV) system. 
The differential dynamics of the model \eqref{eqn:dtsys} are given by
\begin{equation}
\label{syst:ltv}
\delta x (t+1) = {\partial f \over \partial x}(x(t)) \delta x(t),
\end{equation}
with $\delta x \in \rea^n$ representing the infinitesimal displacement. Informally, if the LTV system \eqref{syst:ltv} is exponentially stable along any feasible trajectories $x(t)$, we can say the system \eqref{eqn:dtsys} is contracting. Its formal definition is given as follows.

\begin{definition}
\label{def:contraction}\rm
Given the DT system \eqref{eqn:dtsys}, if there exists a uniformly bounded metric $M(x)$, {i.e.} $a_1 I_n \preceq M(x) \preceq a_2 I_n$ for some $a_2\ge a_1 >0$, guaranteeing
\begin{equation}
\label{cond:contraction}
{\partial f\over \partial x}(x(t))^\top M(x(t+1)) {\partial f\over \partial x}(x(t)) - M(x(t)) \preceq - \beta M(x(t)), 
\end{equation}
with $0< \beta<1$, then we say that the given system is contracting.
\end{definition}

A central result of contraction analysis is that, for contracting systems, all trajectories converge exponentially to a single trajectory, i.e., for any two trajectories $x_a$ and $x_b$
$$
|x_a(t) - x_b(t)| \le a_0\beta^t |x_a(0) - x_b(0)|
$$
for some $a_0>0$.

In this paper, we propose an algorithm to learn dynamical models which are contracting in the sense of Definition \ref{def:contraction}. Generally speaking, verifying or guaranteeing contraction for a nonlinear model is non-trivial. Motivated by the Koopman approach, we instead consider a transformation of the state space into a higher-dimensional manifold on which the dynamics are linear. A natural question that arises is the conservativeness of such an approach for verifying stability. However, it was recently shown in \cite{YiManchester2021equivalence} that the Koopman and contraction approaches are equivalent to each other for stability analysis when considering CT dynamical models. 

\section{Stability Criterion for Discrete-time Koopman Models } \label{sec:koopman_stability}

In this section, we extend the main result in \cite{YiManchester2021equivalence} to DT systems, i.e. the Koopman and contraction approaches are equivalent for nonlinear stability analysis. 
We will show that the model set proposed here can provide sufficient degrees of freedom
for learning nonlinear DT models.

\begin{theorem}
\label{thm:1}\rm
Consider the system \eqref{eqn:dtsys}. Suppose that there exists a mapping $\phi: \rea^n \to \rea^N$ with $N \ge n$ such that
\begin{itemize}
\item[\bf D1] There exists Schur stable matrix $\cala\in \rea^{N\times N}$ satisfying
		\begin{equation}
		\label{AEq}
		    \phi(f(x)) = \cala \phi(x).
		\end{equation}
\item[\bf D2] $\Phi(x):= \nabla \phi(x)^\top$ has full column rank, and $\Phi(x)^\top \Phi(x)$ is uniformly bounded.
\end{itemize}
Then system \eqref{eqn:dtsys} is contracting with the contraction metric $ \Phi(x)^\top P \Phi(x)$, where $P$ is any positive-definite matrix satisfying $P - \cala^\top P\cala \succ 0$. Conversely, if the system \eqref{eqn:dtsys} is contracting with the metric $M(x) \in \rea^{n\times n}_{\succ0}$, and assuming that $f$ is invertible and its inverse $f^{-1}$ is continuous. Then, in any invariant compact set $\calx \subset \rea^n$, there exists a continuous Koopman mapping $\phi: \rea^n \to \rea^n$ verifying {\bf D1} and {\bf D2}.
\end{theorem}


	\begin{proof}
	($\Rightarrow$)
	From {\bf D2} there exists a matrix $P = P^\top \succ 0$ satisfying the Lyapunov condition
	\begin{equation}
	\label{lya_eq}
	    P - \cala^\top P \cala \succ Q,
	\end{equation}
	for some constant positive definite matrix $Q \succ 0$ without loss of generality. We define a new coordinate $z:=\phi(x)$ in which the infinitesimal displacement $\delta z_t \in \rea^{N}$ at time $t$ is given by
	\begin{equation}
	    \delta z_{t} = \Phi(x_t)\delta x_{t},
	\end{equation}
	where $\delta x_t$ is an infinitesimal displacement in the $x$-coordinate.
	
	The DT differential dynamics of $x$ can be written as 
	\begin{equation}
	    \delta x_{t+1} = F(x_t) \delta x_{t},
	\end{equation}
	where $ F(x) := \nabla f(x)^\top$.
	
	Similarly, for $z$ we have
	\begin{equation} \label{eqn:diff_dyn}
	\delta z_{t+1} = \cala\Phi(x_t)\delta x_{t} = \Phi(x_{t+1})  F(x_t) \delta x_{t},
	\end{equation}
	where we have used the relations $z_{t+1} = Az_t$ and $z_{t+1} = \phi(f(x_t))$ in their differential forms. Hence, we obtain
	$$
	\Phi(x_{t+1}) F(x_t) = \cala \Phi(x_t).
	$$
	Due to the full column rank of $\Phi(x)$ and \eqref{lya_eq}, it follows that 
\begin{equation}
\label{phiQphi}
	\Phi(x_t)^\top(P - \cala^\top P\cala)\Phi(x_t) \succ \Phi(x_t)^\top Q \Phi(x_t).
\end{equation}
	Then, by substituting \eqref{eqn:diff_dyn}, 
    we have 
    \begin{equation} \label{eqn:contraction_cond_1}
    \begin{aligned}
            & ~\Phi(x_t)^\top P \Phi(x_t) -  F(x_t)^\top \Phi(x_{t+1})^\top P \Phi(x_{t+1})  F(x_t)
            \\
            \overset{\eqref{eqn:diff_dyn}}{\succ} &~ \Phi^\top Q \Phi \succeq {\lambda_{\tt min}(Q) \over \lambda_{\tt max}(P)}  \Phi^\top P \Phi. 
    \end{aligned}
    \end{equation}
    Now since $\Phi$ has full column rank and $P \succ 0$, we have $M(x) := \Phi^\top P \Phi \succ 0$. Substituting into \eqref{eqn:contraction_cond_1}:
    \begin{equation}
        M(x_t) -  {F}(x_t)^\top M(x_{t+1}){F}(x_t) \succ \beta M(x_t),
    \end{equation}
    with $\beta:={\lambda_{\tt min}(Q)/\lambda_{\tt max}(P)}$. By selecting $Q= \rho P$ with $\rho\in (0,1)$, we have $\beta \in (0,1)$. This is exactly the contraction condition for the system \eqref{eqn:dtsys} with respect to the metric $M$.
    
    ($\Leftarrow$) For the given DT system, from directly applying the Banach fixed-point theorem we conclude that there exists a unique fixed-point $x_\star \in \calx$, i.e. $f(x_\star) = x_\star$.

 First, we parameterise the unknown mapping $\phi(x)$ as
$
\phi(x) := x + T(x),
$
with a new mapping $T(x)$ to be searched for. Then, the algebraic equation \eqref{AEq} becomes
$
T(f(x))+ f(x) = \cala x + \cala T(x).
$
By fixing $\cala = \nabla f(x_\star)^\top$, from the contraction assumption, we have 
$
M(x_\star) - \cala^\top M(x_\star) \cala \succeq \beta M(x_\star),
$
thus $\cala$ being Schur stable. It yields
\begin{equation}
\label{AE:T}
T(f(x)) = \cala T(x) + H(x),
\end{equation}
in which we have defined
$
H(x):= \cala x - f(x).
$
We make the key observation that the algebraic equation \eqref{AE:T} exactly coincides with the one in the formulation of the Kazantzis-Kravaris-Luenberger observer for nonlinear DT systems 
\cite[Eq. (7)]{BRIetal}. In our case, the function $H(x)$ is continuous and, following \cite[Theorem 2]{BRIetal}, we have a feasible solution to \eqref{AE:T} as follows:
\footnote{The second assumption in \cite{BRIetal} holds true in any backward invariant compact set. Since 
contracting systems generally cannot guarantee such invariance, we may modify the dynamics as $x_{t+1} = \rho(x_t)f(x_t)$ with
$$
\rho(x) = \left\{
\begin{aligned}
1, \quad &\mbox{if~~} x\in \mathtt{cl}(\mathcal{X})
\\
0, \quad & \mbox{if ~~} x\notin \mathcal{X}'
\end{aligned}
\right.
$$
with $\calx \subset \calx'$, and then continue the 
analysis.
} 
\begin{equation}
\label{solution}
T(x) = \sum_{j=0}^{+\infty} \cala^i  H(X(x,-j+1)),
\end{equation}
with the definition
$$
X(x,j) = \underbrace{f\circ f\circ \cdots \circ f}_{j~\mbox{\small times}}(x)
, \quad
X(x,-j) = (f^{-1})^j(x)
$$
for $i \in \mathbb{N}_+$.

Although $\phi^0(x) := x+ T(x)$ with $T$ defined above satisfies {\bf D1} in the entire set $\calx$, the condition {\bf D2} may be not true. Hence, we need to modify the obtained $\phi^0(x)$. By considering the evolution of the trajectories in the $x$- and $z:=\phi(x)$-coordinates respectively, we have
$$
	z(t_x) = \phi^0(x(t_x))= \phi^0(X(x,t_x)) = \cala^{t_x} \phi^0(x),
$$
with $t_x\in \mathbb{N}_+$, thus satisfying
$
\phi^0(x) = \cala^{-t_x}\phi^0(X(x,t_x)).
$
Then, we modify $\phi^0(x)$ into
\begin{equation}
\label{phi}
\phi(x) := \cala^{-t_x}[X(x,t_x) + T(X(x,t_x))]
\end{equation}
with a sufficiently large $t_x \in \mathbb{N}_+$. 

Finally, let us check conditions {\bf D1} and {\bf D2}. For the algebraic condition, we have
$$
\begin{aligned}
	\phi(f(x)) & = \cala^{-t_x} \phi^0(X(f(x),t_x))
	\\
	& =  \cala^{-t_x} \phi^0 (f(X(x, t_x)))
	\\
	& 
	= \cala^{-t_x} \cdot \cala\phi^0(X(x,t_x))
	\\
	& 
	= \cala\phi(x)
\end{aligned}
$$ 
where we have used the fact 
$$
X(f(x),t_x) =  \underbrace{f\circ f\circ \cdots \circ f}_{(j+1)~\mbox{\small times}} 
=
f(X(x,t_x))
$$ 
in the second equation. Therefore, $\phi(x)$ defined in \eqref{phi} satisfies the algebraic equation \eqref{AEq}. Regarding {\bf D2}, let us study the Jacobian of $\phi(x)$ in \eqref{phi}, which is given by
$$
{\partial \phi \over \partial x}(x) = \cala^{-t_x} \left[
I + {\partial T \over \partial x}(X(x,t_x) )
\right]
{\partial X \over \partial x}(x).
$$
On the other hand, we have that $\nabla_x X$ is full rank and
$$
H (x_\star) = 0, \quad {\partial H \over \partial x}(x_\star) =0,
$$
as a result $\nabla T(x_\star)=0$. If $t_x \in \mathbb{N}_+$ is sufficiently large, the largest singular value of $\nabla T(X(x,t_x))$ would be very small, and then the identity part of $\phi(x)$ will dominate $\nabla \phi(x)$. Hence, $\phi(x)$ is an injection for a large $t_x\in \mathbb{N}_+$.
\end{proof}

\section{Model Set} \label{sec:model_set}

In this section, we define the model set that we optimize over in our learning framework. We parameterize both the Koopman observables $\phi(x)$ and the matrix $\cala$ in our model and train them jointly. 
To the best of our knowledge, the joint learning of $\phi$ and $\cala$ with the model stability constraint---as done in the paper---has not been previously considered in the literature.

Recall our definition of the Koopman model \eqref{eqn:model}. By Theorem \ref{thm:1}, the model \eqref{eqn:model} is guaranteed to be contracting if Conditions \textbf{D1} and \textbf{D2} are satisfied. In the following, we propose parameterizations of $\phi(x)$ and $\cala$ that satisfy these conditions.

\subsection{Parameterization of observables}

We propose to parameterize the observables as:
\begin{equation}\label{eqn:phi_param}
    \phi(x) = Cx + \varphi(x, \theta_{NN}),
\end{equation}
where $C = [I_n, 0_{n \times (N - n)}]^\top$. The nonlinear part $\varphi(x)$ can be any differentiable function approximator, parameterized by $\theta_{NN}$. For brevity, we drop the dependence on $\theta_{NN}$ in our notation. In this paper, we consider $\varphi(x)$ as a feedforward neural network due to its scalability, but any differentiable function approximator can be used.

The dimensionality of the observables $N$ is a hyperparameter chosen by the user. For $N = n$, the observables will be of the same form as the constructive mapping $\phi^0(x) = x + T(x)$ in Theorem \ref{thm:1}. 

In order to reconstruct the original state $x$ from the observables, we need to train a separate function $\phi^L(z)$ to compute the left-inverse of $\phi(x)$. Indeed, the left invertibility of $\phi$ is necessary for condition {\bf D2}. We propose to simply parameterize this left inverse function as another neural network $\phi^L(z, \theta_{L})$.

\begin{remark}
There are many possible parameterizations of the observables that are compatible with our framework, with Equation \eqref{eqn:phi_param} being just the one chosen to mimic the constructive mapping from Theorem \ref{thm:1}. For some parameterizations, the left inverse may be computed analytically and does not have to be modelled as a separate function. For example, if $\phi(x) = [x^\top, \varphi(x)^\top]^\top $, then the left inverse is simply $x = C\phi(x)$, where $C = [I, 0]$.
\end{remark}

\subsection{Parameterization of the Koopman operator}
The Koopman matrix $\cala$ has to satisfy Condition \textbf{D1} of Theorem \ref{thm:1}, i.e. it must be Schur stable. There are many equivalent conditions for enforcing stability of linear systems, including the well-known Lyapunov inequality $P - \cala^\top P\cala \succ 0$ for some $P \succ 0$, and the recently proposed parameterization in \cite{GillisEtAl2020note}, which was used to train stable Koopman operators for fixed observables in \cite{MamakoukasEtAl2020Memory}. However, solving optimization problems with these constraints in an efficient manner is non-trivial, especially when jointly searching for the observables.

In the following, we present an \textit{unconstrained} parameterization of $\cala$, which is a special case of the direct parameterization approach proposed in \cite{RevayEtAl2021Recurrent}. 

\begin{proposition}
Consider the parametric matrix $\cala(L, R)$ defined as:
\begin{equation} \label{eqn:A_param}
    \cala(L, R) = 2 (M_{11} + M_{22} + R - R^\top)^{-1}M_{21},
\end{equation}
where
\begin{equation} \label{eqn:M}
    M := \begin{bmatrix} M_{11} & M_{12} \\ M_{21} & M_{22} \end{bmatrix} = LL^\top + \epsilon I,
\end{equation}
with $\epsilon$ a small positive constant. Then for any real-valued $L \in \rea^{2N \times 2N}$ and $R \in \rea^{N\times N}$, $\cala_0 = \cala(L, R)$ is a necessary and sufficient condition for $\cala_0$ to be Schur stable.
\end{proposition}
\begin{proof} 
Let $E = (M_{11} + M_{22} + R - R^\top)/2$, $F = M_{21}$ and $P = M_{22}$. Then we have $\cala(L, R) = E^{-1}F$ and 
\begin{equation} \label{eqn:implicit_mat}
	    M = \begin{bmatrix}
	E + E^\top - P & F^\top \\ F & P
	\end{bmatrix}.
\end{equation}
It has been shown that $M \succ 0$ is necessary and sufficient for $E^{-1}F$ to be Schur stable \cite{TobenkinEtAl2017Convex}. Since our parameterization $M = LL^\top + \epsilon I$ is positive definite by construction, this proves sufficiency for $\cala(L, R)$ to be Schur stable. Additionally, all $M$ can be constructed from $L$, e.g. via Cholesky factorization, and by extension, all $E$, $F$ and $P$ can be constructed from $L$ and $R$. This completes the proof.
\end{proof}

\subsection{Overall Koopman Model}
 It may be helpful to think of our model as a linear system with an output $\hat{x}$ that is an estimate of the original state:
\begin{equation} \label{eqn:linear_model}
    \begin{aligned}
    z(t) &= \cala z(t-1), \\
    \hat{x}(t) &= \phi^L(z(t)),
    \end{aligned}
\end{equation}
where $z(0) = \phi(x_0)$. This system is equivalent to \eqref{eqn:model}. In this form, it is clear that as long as $\cala$ is stable and $\phi^L$ is uniformly bounded, then the output $\hat x$ will always converge to a single equilibrium.

To summarize, our model parameters consist of:
\begin{equation} \label{theta}
    \theta = \{ \theta_{NN}, \theta_{L}, L, R \}.
\end{equation}

\section{Learning Framework} \label{sec:learning}
\subsection{Optimization Problem} \label{subsec:optim}

To fit the model parameters \eqref{theta} to data, we consider the problem of minimizing the \textit{simulation error} in the embedding space:
\begin{equation} \label{eqn:sim_error}
J_{se} := \frac{1}{T} \sum_{t=0}^{T}\vecnorm*{\tilde{z}_{t} - {z}_t}^2_2,
\end{equation}
where $\tilde{z_t} = \phi(\tilde{x}_t)$, and $z_t = \cala^t \phi(\tilde x_0)$. While we could also minimize the simulation error in $x$, in practice we found this produced poor results. 

The complete optimization problem is: 
\begin{equation} \label{eqn:min_J_se}
    \min_{\theta \in \Theta} \frac{1}{T} \sum_{t=0}^T \vecnorm*{\phi(\tilde{x}_t) - \cala(L, R)^t \phi(\tilde{x}_0)}^2_2 + \alpha J_{rec}.
\end{equation}
The reconstruction loss $J_{rec}$ is defined as
\begin{equation} \label{eqn:recon_error}
    J_{rec} = \frac{1}{T} \sum_{t=0}^T \vecnorm*{\tilde x_t - \phi^L(\phi(\tilde x_t))}^2_2.
\end{equation}
Minimizing $J_{rec}$ gives us an approximate left-inverse $\phi_L$ for the Koopman mapping. The loss $J_{rec}$ can be thought of as a penalty term that relaxes the constraint
\begin{equation*}
    x = \phi^L(\phi(x)) \ \forall x,
\end{equation*}
and the constant $\alpha$ is a hyperparameter that determines the weighting of the penalty.

We emphasize two important properties of Problem \eqref{eqn:min_J_se}. First, it is an unconstrained optimization problem. The parameter set $\Theta$ is the space of real numbers of the appropriate dimensionality. Second, there exists a differentiable mapping from the parameters $\theta$ to the objective for any choice of differentiable mapping $\phi_\theta$, e.g. using our parameterization \eqref{eqn:phi_param} with $\varphi_\theta$ as a neural network.

These two properties enable us to find a local optimum to Problem \eqref{eqn:min_J_se} using any off-the-shelf first-order optimizer in conjunction with an automatic differentiation (autodiff) toolbox. This significantly simplifies the implementation of our framework. Using an autodiff software package, one only needs to write code that evaluates the objective function at each iteration of the optimization process, and the gradients w.r.t. $\theta$ are automatically computed via the chain rule. In contrast, constrained problems such as the one proposed in \cite{MamakoukasEtAl2020Memory} require specialized algorithms to solve. Although the objective \eqref{eqn:min_J_se} is nonconvex, deep learning methods have been shown to be effective at finding approximate global minima for such problems; see \cite[Chapter 21]{Roughgarden2020Beyond} for example.

\begin{remark}
It is worth noting that our model class is agnostic to the optimization problem. In fact, the model can be optimized for any differentiable objective function. This is another advantage of an unconstrained parameterization. 
\end{remark}

\subsection{Implementation Details}
We implemented our learning framework in PyTorch\footnote{https://github.com/pytorch/pytorch} and used the Adam optimizer \cite{KingmaBa2014Adam} to solve Problem \eqref{eqn:min_J_se}. The neural network parameters $\theta_{NN}$ and $\theta_L$ are initialized using the default scheme in PyTorch, while $L$, $R$, and $b$ are initialized randomly from a uniform distribution. 

\subsubsection{Fast matrix power computation} \label{sec:fast_mat}
As explained in Section \ref{subsec:optim}, the only code we need to implement for solving Problem \eqref{eqn:min_J_se} is the evaluation of the objective function, which is also the main computational bottleneck. In particular, repeatedly computing the matrix power $\cala^t$ for the same $\cala$ and many $t$'s can be computationally inefficient.
Here we describe a simple trick to speed up matrix power computations. Consider the eigendecomposition of $\cala$ given by $V\Lambda V^{-1}$, where the columns of $V$ are the eigenvectors and $\Lambda$ is a diagonal matrix of the eigenvalues. Then it is clear that
\begin{equation}
    \cala^t = (V\Lambda V^{-1})^t = V \Lambda^t V^{-1}
\end{equation}
for integer $t$. Notice that $\Lambda^t$ can be computed element-wisely for each eigenvalue on the diagonal, which offers a significant speed-up over computing a matrix power. This trick assumes $\cala$ is diagonalizable, but this can easily be verified in code and, if the condition is not satisfied, the original matrix power computation can be performed instead.

\section{Continuous-time Case} \label{sec:continuous}
In this section, we briefly present the CT formulation of our learning framework. For an autonomous system governed by an ordinary differential equation (ODE) $\dot{x} = f(x)$, there exists a semigroup of Koopman operator $\mathcal{K}^t$ associated with the flow map $X(x, t)$ of the system, defined as: 
\begin{equation}
    \mathcal{K}^t \phi(x(t)) = \phi(X(x, t)).
\end{equation}
We refer to the infinitesimal generator of this semigroup as the continuous-time Koopman operator $\tilde{\mathcal{K}}$ \cite{WilliamsEtAl2015Data}:
\begin{equation} \label{eqn:ctkoopman}
\tilde{\mathcal{K}}\phi(x(t)) = \frac{d}{dt}\phi(x(t)) = \nabla \phi \cdot f(x(t)).
\end{equation}

\begin{definition}[CT Koopman model]
The continuous-time Koopman model is given by:
\begin{equation}
    x(t) = \phi^L(\exp(A_\theta t)\phi(x_0)),
\end{equation}
where the Koopman mapping $\phi$ is parameterized as in \eqref{eqn:phi_param}, and the finite-dimensional matrix $A_\theta$ is parameterized as:
\begin{equation}
    A_\theta = (NN^\top + \epsilon I)^{-1}(-QQ^\top - \epsilon I + {1\over 2} (R - R^\top) ),
\end{equation}
with parameters $N$, $Q$ and $R$. This is an unconstrained parameterization of all CT stable (Hurwitz) matrices.
\end{definition}

Given full-state trajectory data $\lbrace \tilde{x}_k \rbrace^{K}_{k=0}$ with corresponding time $\lbrace t_k \rbrace^{K}_{k=0}$, we would like to minimize the simulation error of the Koopman model. The optimization problem is
	\begin{equation} \label{eqn:cont_sim_error}
	    \min_{\theta \in \Theta} \frac{1}{K} \sum_{k=0}^{K} \vecnorm*{\phi_\theta(\tilde x_k) - \exp(A_\theta t_k)\phi_\theta(\tilde x_0)}^2_2 + \alpha J_{rec},
	\end{equation}
where $J_{rec}$ is as defined in Equation \eqref{eqn:recon_error}.
Problem \eqref{eqn:cont_sim_error} is an unconstrained optimization problem just like the discrete-time problem, hence a local minimum can be obtained using a first-order optimizer and an autodiff software package. A trick similar to that described in Section \ref{sec:fast_mat} can be used to compute the matrix exponential in the objective in \eqref{eqn:cont_sim_error}. 

Note that in terms of the data required, the only difference between the DT and CT learning frameworks is that the CT case requires the time corresponding to each data point. The CT problem can be useful to consider when the data is sampled at non-uniform time intervals, or when the sampling rate differs between the training and test scenarios.

\section{Numerical Examples} \label{sec:examples}

\begin{table}
\centering
\caption{Comparison of model sets for our method and prior works.}
\begin{tabular}{|M{1.5cm}|M{2cm}|M{2cm}|M{1.2cm}|} 
    \hline
    Method & Learns observables or eigenfunctions & Continuous or discrete time & Stability constraint \\
    \hline
    SOC \cite{MamakoukasEtAl2020Memory} & Neither & Discrete & \cmark \\
    LKIS \cite{TakeishiEtAl2017Learning} & Observables & Discrete & \xmark \\
    \cite{LuschEtAl2018Deep} & Eigenfunctions & Discrete & \xmark \\
    \cite{PanDuraisamy2020Physics} & Eigenfunctions & Continuous & \cmark \\
    Ours & Observables & Both & \cmark \\
    \hline
\end{tabular}
\label{table:compare_methods}
\end{table}

We validated our framework on the LASA handwriting dataset \cite{KhansariBillard2011Learning}, which consists of human-drawn trajectories of various letters and shapes\footnote{https://cs.stanford.edu/people/khansari/download.html}. It has been widely used as a benchmark for learning contracting dynamics in continuous-time \cite{BlocherEtAl2017Learning, KhansariBillard2011Learning,MohammadKhansari-Zadeh2014, NeumannEtAl2013Neural, RavichandarEtAl2017Learning}. In our results, we trained discrete-time models in order to compare them with existing DT Koopman learning frameworks. Contraction is an important constraint for this data set as unconstrained models can have spurious attractors \cite{KhansariBillard2011Learning}, leading to poor generalization to unseen initial conditions.

For each shape in the dataset, we attempted to train a discrete-time model that would regulate to the desired equilibrium point from any initial condition. To prepare the data for learning DT models, we fitted splines to the trajectories and re-sampled the datapoints at a uniform time interval. The state vector was chosen to be $\tilde x_t = [y_t^\top, \dot y_{t}^\top]^\top \in \rea^4$, where $y_t$ and $\dot y_t$ are the position and velocity vectors at time $t$. All data was scaled to the range $[-1, 1]$ before training. For each shape in the dataset, we performed leave-one-out cross validation. Test trajectories are plotted in Figure \ref{fig:lasa_simulations} as solid black lines for a subset of the shapes in the dataset. 

The metric we used to compare different methods was normalized simulation error (NSE), defined as:
\begin{equation}
    NSE = \frac{\sum_{t=0}^T \vecnorm*{\hat x_t - \tilde x_t}^2_2}{\sum_{t=0}^T \vecnorm*{\tilde x_t}^2_2},
\end{equation}
where $\{\hat x\}^T_{t=0} $ is the simulated trajectory using the learned model, and $\{\tilde x\}^T_{t=0}$ is the true trajectory. The aim of our comparisons was to evaluate our framework against prior methods for learning Koopman models. The key differences of some recent frameworks are summarized in Table \ref{table:compare_methods}. We did not compare against \cite{LuschEtAl2018Deep} as they assume some prior knowledge about the spectrum of the system, which differs from the problem setting we consider. Due to space constraints, we leave comparisons of CT frameworks to future work.

In the following, we refer to our framework as SKEL (Stable Koopman Embedding Learning). 

\subsection{Comparison with other Koopman matrix parameterizations}

We compared SKEL against two recently-proposed Koopman learning frameworks, namely SOC \cite{MamakoukasEtAl2020Memory} and LKIS \cite{TakeishiEtAl2017Learning}. In particular, we compared our unconstrained stable parameterization of the Koopman operator against a constrained stable parameterization (SOC), and a unconstrained parameterization without stability guarantees (LKIS). 

The SOC parameterization is given by $\cala = S^{-1}OCS,$
where $O$ is orthogonal and $C$ is positive-semidefinite with $\norm{C} \leq 1$. A projected gradient descent method was used to solve the optimization problem. The LKIS parameterization is
$    \cala = Y_1 Y_2^\dag,$
where $Y_1$ and $Y_2$ are as defined in Equation \eqref{eqn:dmd_solution}, with parametric $\phi(x)$.

To make it a fair comparison, we kept all other aspects of the optimization problem the same, i.e. using simulation error as the optimization objective and using parametric observables of the form \eqref{eqn:phi_param}. We were interested mainly in comparing parameterizations of the Koopman operator as our framework is agnostic to choice of objective and observables, and these choices often depend on the particular application.

All instances of $\varphi(x)$ were fully-connected feedforward neural networks with ReLU (rectified linear units) activation functions, 2 hidden layers with 50 nodes each and an output dimensionality of 20. Hyperparameter values were chosen to be $\alpha = 10^3$ and $\epsilon = 10^{-8}$.

A boxplot of the normalized simulation error for the three methods is shown in Figure \ref{fig:skel_comparison}. It is clear that SKEL achieves the lowest median NSE on the test set with 95\% confidence. From Figure \ref{fig:losses}, it can be seen that LKIS actually attains the lowest training error, but does not generalize to the test set as well as SKEL. This can be seen as a symptom of overfitting, and shows that the stability guarantees of SKEL have a regularizing effect on the model. With regards to SOC, we observed that the constrained optimization problem would often converge to poor local minima, which is reflected in the relatively high training and test errors.

\begin{figure}
    \centering
    \includegraphics[width=0.45\textwidth]{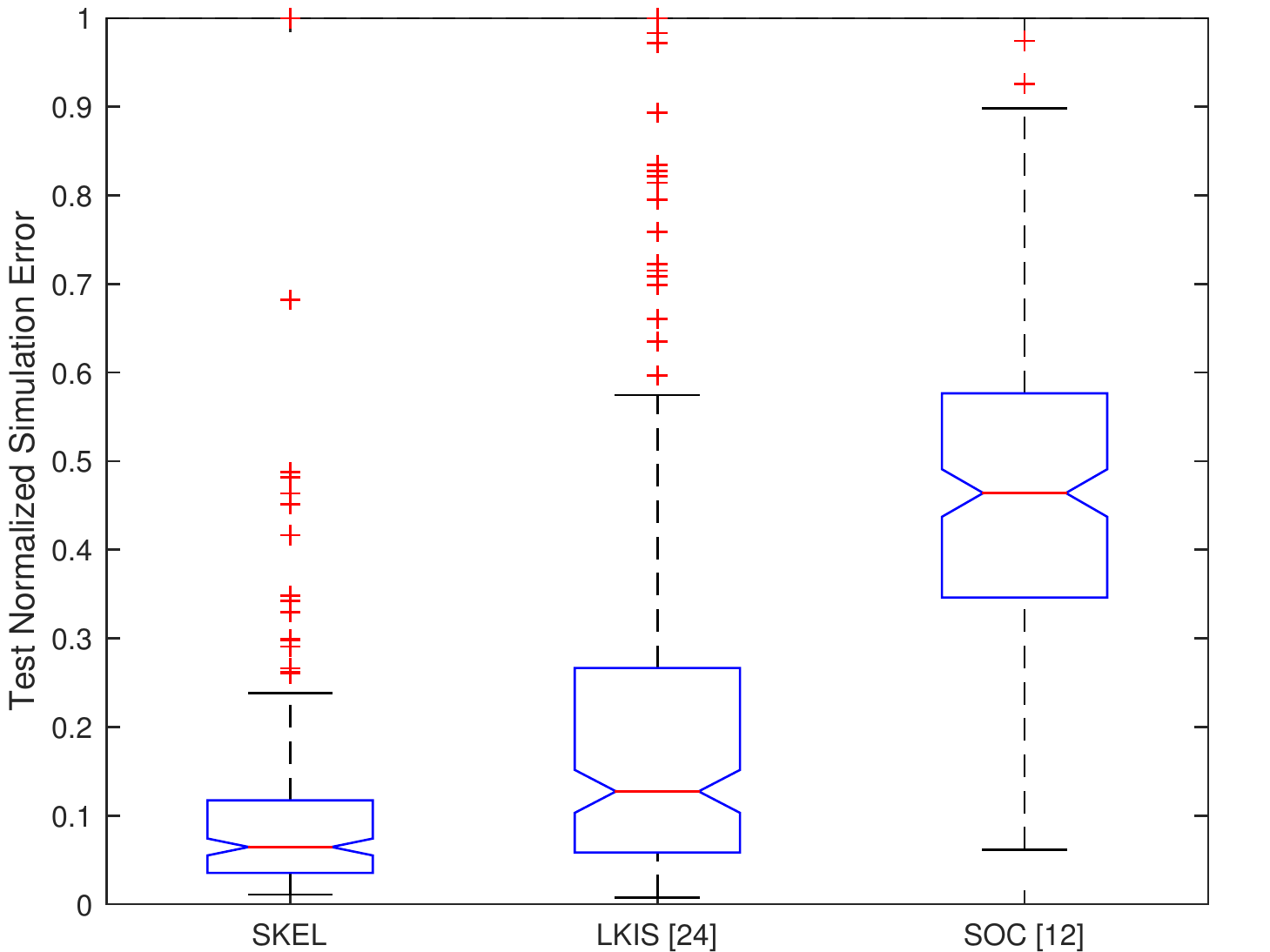}
    \caption{Comparison of SKEL with other Koopman learning methods. Outliers were clipped for better visibility of boxes. Number of outliers with NSE $>1$ from left to right: 1 (SKEL), 15 (LKIS), 0 (SOC).}
    \label{fig:skel_comparison}
\end{figure}

\begin{figure}
    \centering
    \includegraphics[width=0.45\textwidth]{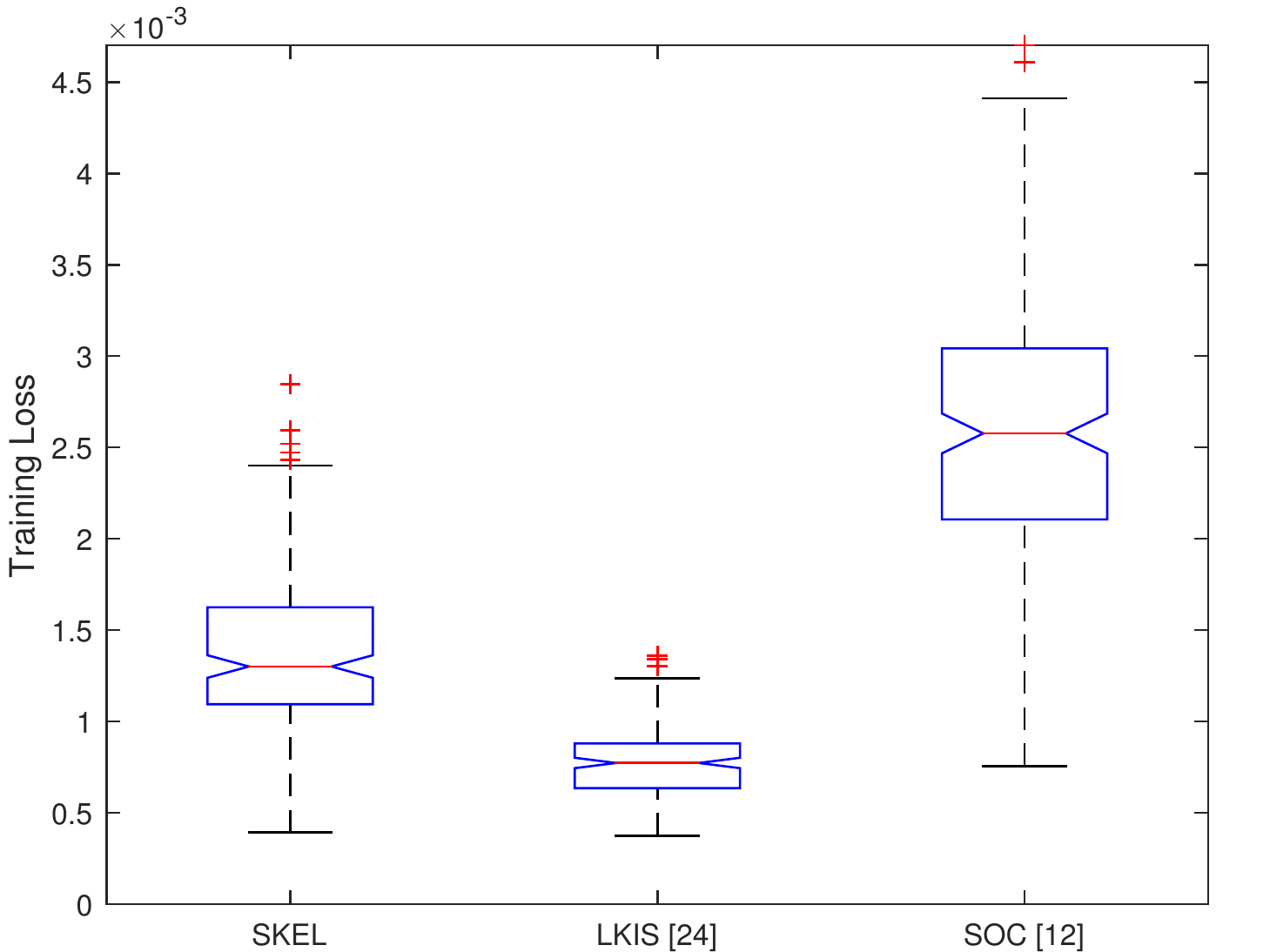}
    \caption{Training loss (Eq. \eqref{eqn:min_J_se}) for each method}
    \label{fig:losses}
\end{figure}

\begin{figure*}
    \centering
    \begin{subfigure}{0.48\textwidth}
        \centering
        \includegraphics[width=\textwidth]{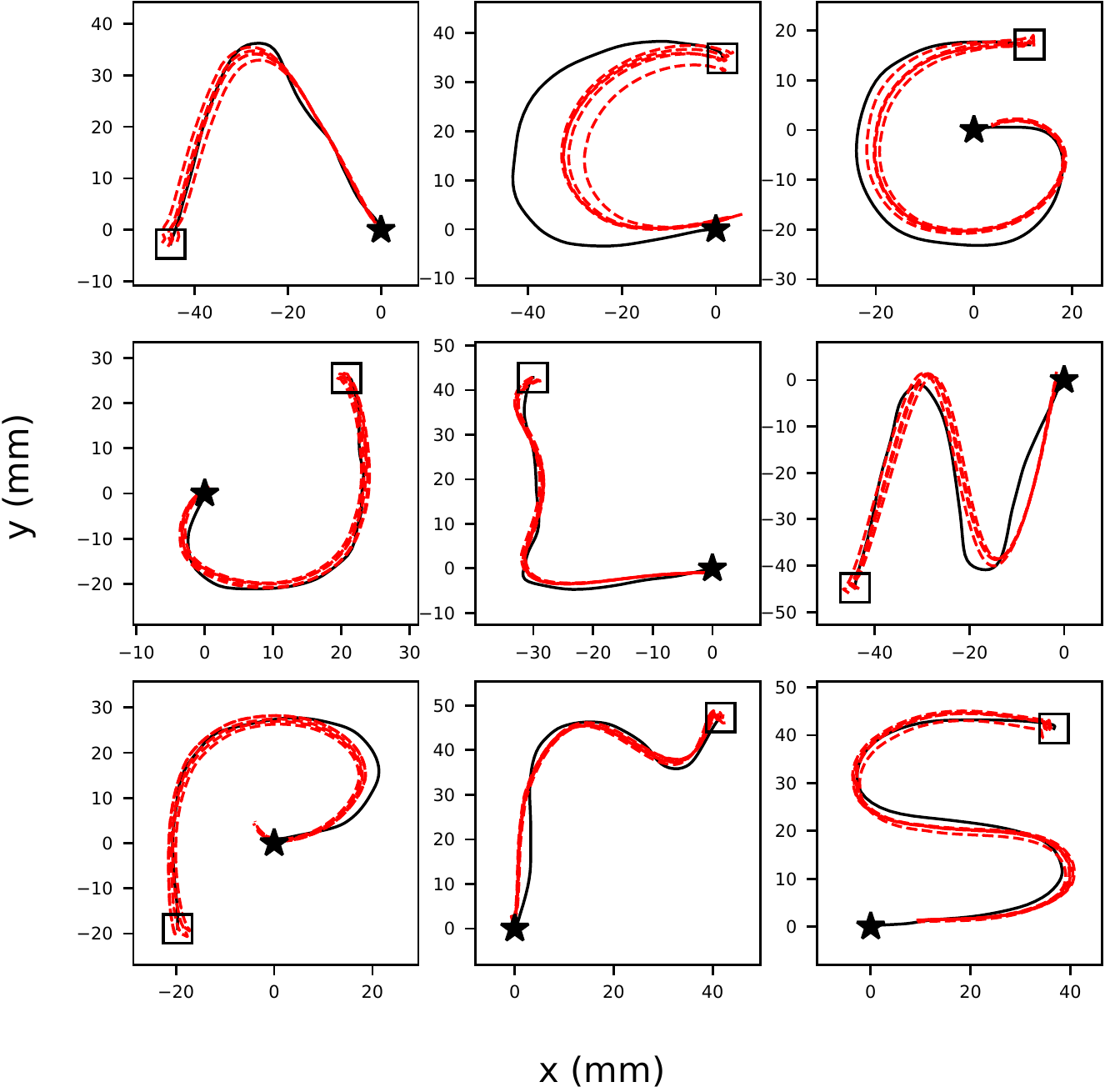}
        \caption{SKEL (ours)}
        \label{fig:skel_sim}
    \end{subfigure}
    \hfill
    \begin{subfigure}{0.48\textwidth}
        \centering
        \includegraphics[width=\textwidth]{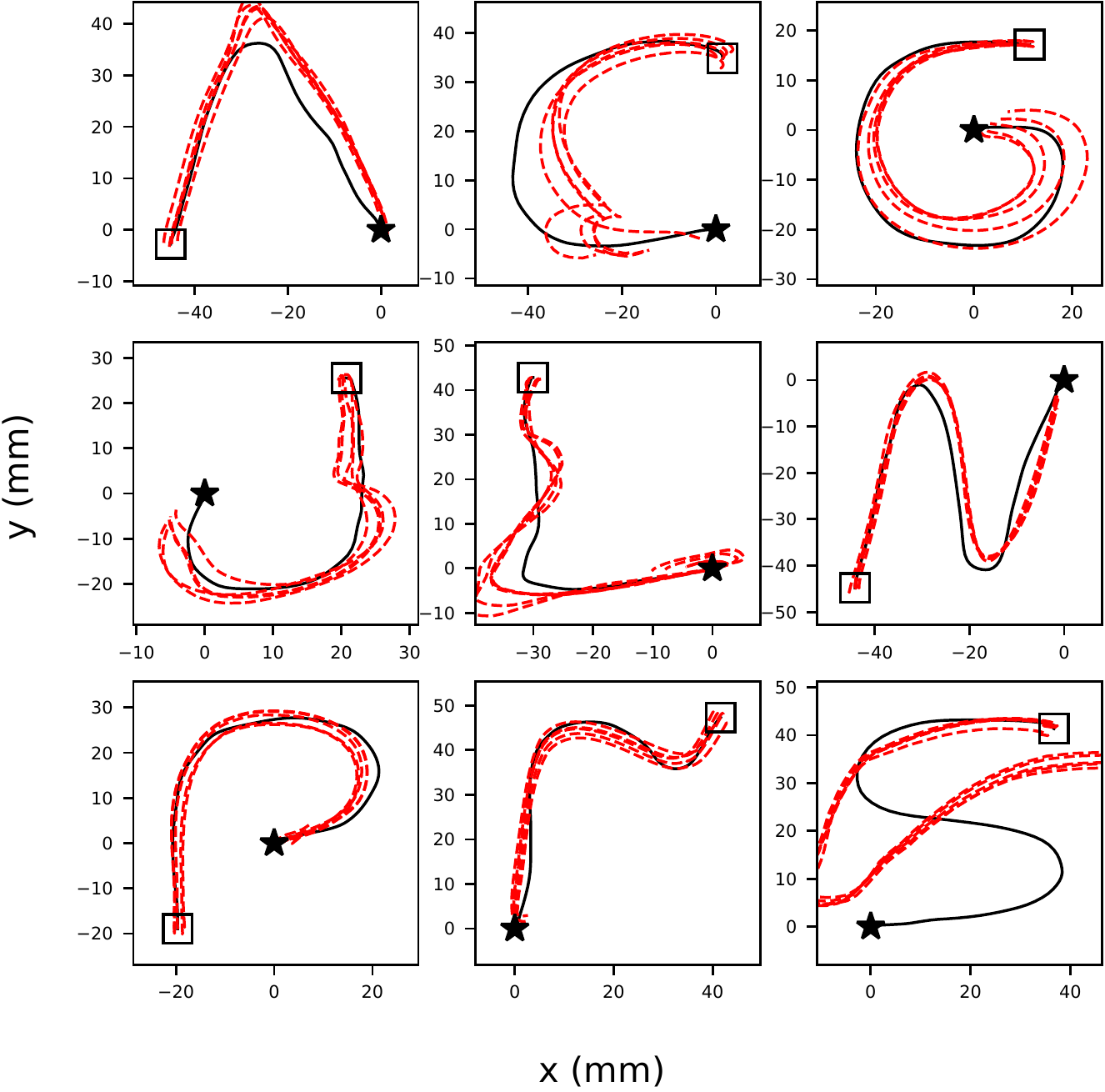}
        \caption{LKIS \cite{TakeishiEtAl2017Learning}}
        \label{fig:lkis_sim}
    \end{subfigure}
\caption{Simulations of SKEL and LKIS models on test data. Trajectories from the models are shown 
as red dotted lines, while the true trajectory is shown as a solid black line. Initial conditions were sampled from a square region of width 2mm centered at the start point of the true trajectory. The target point is marked by a black star.}
\label{fig:lasa_simulations}
\end{figure*}

\subsection{Robustness to perturbations in initial condition}
We performed a qualitative evaluation of the robustness of the models to small perturbations in the initial condition of the test trajectory. We compared only SKEL and LKIS as it was clear from Figure \ref{fig:skel_comparison} that SOC underperformed in this setting. The results are plotted in Figure \ref{fig:lasa_simulations}. It can be seen that the SKEL models produce trajectories that converge to each other due to their contracting property, whereas the LKIS models behave unpredictably, indicating instability of the learned model.

\section{Conclusion}
We have presented a novel Koopman learning framework that jointly models the Koopman operator and observables while guaranteeing model stability, via an unconstrained optimization problem. We showed that our framework outperforms existing Koopman methods on a real-world handwriting problem and achieves the lowest median simulation error. Further work can be done to extend this framework to controlled systems.

\bibliography{stable-koopman-bibtex}
\bibliographystyle{abbrv}

\end{document}